\documentclass[twoside,11pt]{article}

\usepackage{fullpage}
\usepackage{epsf}
\usepackage{fancyhdr}
\usepackage{graphics}
\usepackage{graphicx}
\usepackage{psfrag}
\usepackage[T1]{fontenc}
\usepackage{color}
\usepackage{amsthm}
\usepackage{times}
\usepackage{amsfonts}
\usepackage{amsmath}
\usepackage{amssymb}
\usepackage{mathrsfs}
\usepackage{bm}
\usepackage{accents}
\usepackage{listings}
\usepackage{caption}
\usepackage{mleftright}
\usepackage{subcaption}
\usepackage[linesnumbered, ruled, vlined]{algorithm2e}
\usepackage[titletoc,toc,title]{appendix}
\usepackage{enumerate}
\usepackage{verbatim} 
\usepackage{csquotes} 
\usepackage{upquote} 
\usepackage[bottom]{footmisc} 
\usepackage{thmtools}
\usepackage[dvipsnames]{xcolor}
\usepackage[english]{babel} 
\usepackage[shortlabels]{enumitem}
\usepackage[colorlinks,linkcolor = RawSienna, urlcolor  = RedViolet, citecolor = RoyalBlue, anchorcolor = ForestGreen,bookmarks=False]{hyperref}

\newlength{\widebarargwidth}
\newlength{\widebarargheight}
\newlength{\widebarargdepth}

\makeatletter
\long\def\@makecaption#1#2{
        \vskip 0.8ex
        \setbox\@tempboxa\hbox{\small {\bf #1:} #2}
        \parindent 1.5em  
        \dimen0=\hsize
        \advance\dimen0 by -3em
        \ifdim \wd\@tempboxa >\dimen0
                \hbox to \hsize{
                        \parindent 0em
                        \hfil 
                        \parbox{\dimen0}{\def\baselinestretch{0.96}\small
                                {\bf #1.} #2
                                } 
                        \hfil}
        \else \hbox to \hsize{\hfil \box\@tempboxa \hfil}
        \fi
        }
\makeatother

\makeatletter
\newcounter{manualsubequation}
\renewcommand{\themanualsubequation}{\alph{manualsubequation}}
\newcommand{\startsubequation}{%
  \setcounter{manualsubequation}{0}%
  \refstepcounter{equation}\ltx@label{manualsubeq\theequation}%
  \xdef\labelfor@subeq{manualsubeq\theequation}%
}
\newcommand{\tagsubequation}{%
  \stepcounter{manualsubequation}%
  \tag{\ref{\labelfor@subeq}\themanualsubequation}%
}
\let\subequationlabel\ltx@label
\makeatother

\renewcommand{\baselinestretch}{1.04} 
\date{}
\usepackage[style = alphabetic,citestyle=alphabetic,maxbibnames=99,backend=biber,sorting=nyt,natbib=true,backref=true]{biblatex}
\DefineBibliographyStrings{english}{%
  backrefpage = {Cited on page},
  backrefpages = {Cited on pages},
}
\makeatletter
\renewenvironment{proof}[1][\proofname]{%
  \par\pushQED{\qed}\normalfont%
  \topsep6\p@\@plus6\p@\relax
  \trivlist\item[\hskip\labelsep\bfseries#1\@addpunct{.}]%
  \ignorespaces
}{%
  \popQED\endtrivlist\@endpefalse
}
\makeatother

\newtheorem{theorem}{Theorem}[section]
\newtheorem{lemma}[theorem]{Lemma}

\newtheorem{definition}[theorem]{Definition}

\newcommand{\sign}{\mathrm{sign}}
\renewcommand{\ss}{\subseteq}

\newcommand{\cN}{{\cal N}}

\newcommand{\R}{\mathbb{R}}

\newcommand{\N}{\mathbb{N}}

\newcommand{\diag}{\mathrm{diag}}

\newcommand{\E}{\mathbb{E}}

\renewcommand{\Pr}{\mathbb{P}}
\newcommand{\lv}{\lVert}
\newcommand{\rv}{\rVert}

\renewcommand{\epsilon}{\varepsilon}
\renewcommand{\ln}{\log}
\DeclareSymbolFont{extraup}{U}{zavm}{m}{n}
\DeclareMathSymbol{\varheart}{\mathalpha}{extraup}{86}
\DeclareMathSymbol{\vardiamond}{\mathalpha}{extraup}{87}

\DeclareMathOperator*{\argmin}{arg\,min}

\newcommand{\BP}{\mathsf{A}_\mathsf{BP}}

\newcommand{\htheta}{\widehat{\theta}}

\renewcommand{\epsilon}{\varepsilon}

\renewcommand{\ln}{\log}

\begin{document}

\title{\textbf{Foolish Crowds Support Benign Overfitting}}

\author{
Niladri S. Chatterji \\ 
Computer Science Department \\
Stanford University \\
niladri@cs.stanford.edu \\
      \and
Philip M. Long  \\
Google \\
plong@google.com 
}
\date{\today}

\maketitle

\begin{abstract}
We prove a lower bound on the excess risk of sparse interpolating procedures for linear
regression with Gaussian data in the overparameterized regime. 
We apply 
this result
to obtain a lower bound for basis pursuit 
(the minimum $\ell_1$-norm interpolant)
that implies that its
excess risk can converge at an exponentially slower rate than 
$\mathsf{OLS}$ (the minimum $\ell_2$-norm interpolant), 
even when 
the ground truth is sparse.  
Our analysis exposes the benefit of an effect analogous
to the 
``wisdom of the crowd'', 
except here the
harm arising from fitting the {\em noise} is ameliorated by spreading
it among many directions---the variance reduction
arises from a {\em foolish} crowd.
\end{abstract}

\section{Introduction} \label{s:introduction}

Recently, there has been a surge of interest in benign overfitting,
where a learning algorithm generalizes well despite interpolating
noisy data \citep[see, e.g.,][]{zhang2016understanding,belkin2019reconciling,bartlett2020benign,belkin2021fit}.  Arguably the most basic setting in which this
has been analyzed theoretically is linear regression with Gaussian
data, where upper and nearly matching lower bounds have been
obtained for the
ordinary least squares ($\mathsf{OLS}$) estimator, 
which chooses a parameter vector
$\htheta$ to minimize $\lv \htheta \rv_2$ from among interpolating 
models~\citep{bartlett2020benign,negrea2020defense,tsigler2020benign}.  
Bounds have also been obtained for {\em basis pursuit}~\citep{chen2001atomic},
which minimizes $\lv \htheta \rv_1$ from among interpolating 
models~\citep{muthukumar2020harmless,ju2020overfitting,chinot2020robustness,koehler2021uniform}.  

The upper bounds for the $\mathsf{OLS}$ estimator show that it rapidly approaches the
Bayes risk when the structure of the covariance matrix $\Sigma$ of the inputs is favorable. 
Informally, the following are necessary and sufficient:
\begin{itemize}
    \item the sum of all of the eigenvalues of $\Sigma$ is not too big;
    \item after excluding a few of the largest eigenvalues, there are many small eigenvalues of roughly equal magnitude.
\end{itemize}
A canonical example of such a benign covariance matrix is 
$\Sigma_{k,\epsilon} := \diag(\overbrace{1,\ldots,1}^k,\overbrace{\epsilon,\ldots,\epsilon}^{p-k})$. 
Consider a case where $k \ll p$, the rows of $X$ are $n$ i.i.d.\ draws from
$\cN(0, \Sigma_{k,\epsilon})$,
and, for independent noise $\xi \sim \cN(0, I)$ and a unit-length
$\theta^\star$, $y = X \theta^\star + \xi$. Then a high probability upper bound on the excess risk of the $\mathsf{OLS}$ is proportional to
\[
\frac{k}{n} + \frac{\epsilon p}{n} + \frac{n}{p},
\]
and a nearly matching lower bound is also known \citep{bartlett2020benign,negrea2020defense,tsigler2020benign}.
As an example, if $p = n^2$, $\epsilon = 1/n^2$ and $k$ is a constant, this
is proportional to $1/n$.

If, in addition, $\theta_i^\star = 0$ for $i > k$, upper bounds
are also known for basis pursuit in this setting~\citep{koehler2021uniform}.
On the other hand, they are much worse, scaling
with $p$ like $\frac{1}{\log p}$, and requiring that $p$ be an exponentially
large function of $n$ to converge.

In this paper, we prove a lower bound for basis pursuit in this setting that scales with
$p$ like $\frac{1}{\log^2 p}$.  Our lower bound requires that
$\sigma \geq c \lv \theta^* \rv_2$ for an arbitrarily small positive constant
$c$, along with a few mild technical conditions, including that
$p > n$, so that interpolation is possible, and that $n \gg k$ (see Theorem~\ref{t:bp} for the details).  Note that
$\mathsf{OLS}$ converges much faster than basis pursuit in this setting
despite the fact that $\theta^*$ is sparse.

The lower bound is a special case of a more general bound, Theorem~\ref{t:main}, which can be paraphrased as follows.  Under
the same conditions as Theorem~\ref{t:bp}, including 
$\sigma \geq c \lv \theta^* \rv_2$, any interpolating procedure that, with high probability,
outputs an $s$-sparse model $\htheta$, must suffer excess loss 
proportional to $\frac{\sigma^2 n}{s \log^2 p}$.  We then get Theorem~\ref{t:bp}
by showing that, in this setting, basis pursuit almost surely outputs an
$n$-sparse model.

This analysis sheds light on why the overfitting of $\mathsf{OLS}$
is so benign.  Because $\mathsf{OLS}$ interpolates, we can think of
the parameters of its output as storing the noise---$\mathsf{OLS}$ benefits from spreading the noise evenly among many
parameters, 
where each small fragment of noise has a
tiny effect.   This phenomenon is akin to the reduction
in variance arising from prediction using a
weighted average of many covariates
that is commonly referred to as the ``wisdom of the crowd''~\citep{surowiecki2005wisdom}.
Here, by spreading the harm arising from
fitting the {\em noise} among many parameters, the algorithm 
benefits 
from a {\em foolish} crowd.

\paragraph{Additional related work.}  \citet{muthukumar2020harmless}
established the same lower bound for basis pursuit in a setting with isotropic
covariates
\citep[see also,][]{ju2020overfitting} in the case that
$\theta^* = 0$.
Accommodating the possibility that 
$\theta^* \neq 0$ complicates the argument a bit, 
but
our main contribution is to demonstrate slow
convergence  
for basis pursuit
in settings where $\mathsf{OLS}$
enjoys fast convergence.  
\citet{chinot2020robustness}
proved an upper bound on the risk of basis pursuit, but,
as pointed out by~\citet{koehler2021uniform}, it does not
imply a bound on the excess risk.  Limitations of algorithms
that output sparse linear classifiers have also been studied
previously~\citep{helmbold2012necessity}.

After a preliminary version of this work was posted on arXiv~\citep{chatterji2021foolish}, some related work was published~\citep{ wang2021tight, li2021minimum, donhauser2022fast} that established upper bounds on the excess risk of the minimum $\ell_1$-norm interpolator. In particular, \citet{wang2021tight} showed an upper bound on the excess risk, that in the case with isotropic covariates scales as $\sigma^2/\log(p)$, almost matching our lower bound. 

\section{Preliminaries}\label{s:prelim}

For $p, n \in \N$, an {\em example} is a member of $\R^p \times \R$, and
a {\em linear regression algorithm} takes as input $n$ examples, and outputs
$\htheta \in \R^p$.  
For a joint probability distribution $P$ over $\R^p \times \R$,
the {\em excess risk} of $\htheta$ with respect to
$P$
is
\begin{align*}
R(\htheta):=\E_{(x,y) \sim P} [(\htheta \cdot x - y)^2]
 - \inf_{\theta^\star} \E_{(x,y) \sim P} [(\theta^\star \cdot x - y)^2].
\end{align*}
We refer to the following as the
$(k,p,n,\epsilon,\sigma)$-scenario: 
\begin{itemize}
    \item $X \in \R^{n\times p}$ is a matrix whose rows are i.i.d.\ draws from
     $\cN(0, \diag(\overbrace{1,1,\ldots,1}^k, \overbrace{\epsilon,\ldots,\epsilon}^{p-k}))$, and
    \item 
    for $\theta^\star \in \R^p$ with 
    $(\theta^\star_{k+1},\ldots,\theta^\star_p) = 0$,
        and $\xi \sim \cN(0, \sigma^2 I)$, $y = X \theta^\star + \xi$.
\end{itemize}
For $\delta > 0$, $s,n \in \N$ and 
a joint probability distribution $P$ over $\R^p \times \R$,
we say that a regression algorithm $\mathsf{A}$ is an $(s, \delta)$-sparse interpolator for $P$ and $n \in \N$ if it satisfies the following: With probability $1 - \delta$ over the independent draw of the samples
$(x_1,y_1),\ldots,(x_n,y_n)\sim P$, the output $\htheta$ of $\mathsf{A}$
\begin{itemize}
    \item interpolates the data (that is, satisfies $\htheta \cdot x_1 = y_1,\ldots,\htheta \cdot x_n = y_n$), and
    \item has at most $s$ non-zero components.
\end{itemize}
Given $X \in \R^{n \times p}$ and $y \in \R^n$, the {\em basis pursuit} algorithm (minimum $\ell_1$-norm interpolant)
$\BP$
outputs
\begin{align*}
\argmin_{\theta} \lv \theta \rv_1, \quad \mbox{s.t. } X\theta=y,
\end{align*}
if there is such a $\theta$, and otherwise behaves arbitrarily
(say outputting $0$).
\paragraph{Notation.} For any $j \in \mathbb{N}$, we denote the set $\{ 1,\ldots,j \}$ by $[j]$. Given a vector $v$, let $\lv v \rv_2$ denote its Euclidean norm and $\lv v \rv_1$ denote its $\ell_1$-norm. Given a matrix $M$, let $\lv M\rv_{op}$ denote its operator norm. For $z \in \R$, we denote $\max \{ z, 0 \}$ by $[ z ]_+$.

\section{Main results} \label{s:main_results}
We are ready to present our main result, a high probability lower bound on the excess risk for any $(s,\delta)$-sparse interpolator.
\begin{theorem}
\label{t:main}
For any $0<c_1 \le 1$,
there are
 absolute positive constants
$c_2, c_3$
such that
the following holds.
For any $0 \leq \delta \leq c_2$, 
for any
$(k,p,n,\epsilon,\sigma)$ 
such that $\sigma \geq c_1 \lv \theta^\star \rv_2$, $p \geq n + k$, 
$n \geq \log^2(1/\delta) + k^{1+c_1}$,
for any $n \leq s \leq p - k$,
and any regression algorithm $\mathsf{A}$ that is an $(s, \delta)$-sparse interpolator for the 
$(k,p,n,\epsilon,\sigma)$-scenario $P$, 
with probability
$1 -4 \delta$ over 
$n$ random draws from $P$, the output $\htheta$ of $A$ satisfies
\begin{align*}
    R(\htheta) \geq \frac{c_3 \sigma^2 n}{s \log^2 (3p/s)}.
\end{align*}
\end{theorem}
This theorem shows that a sparse interpolating predictor suffers large excess risk. Intuitively, the proof follows since an $s$-sparse interpolating predictor needs to hide the ``energy'' of the noise, which roughly scales like $\sigma^2 n$, in just $s$ coordinates. 
If it attempts to hide it in the first $k$ coordinates, then it suffers from large bias. 
If it hides it in the tail, then it suffers large variance. 

Next, we state our result for basis pursuit, the minimum $\ell_1$-norm interpolator.
\begin{theorem}
\label{t:bp}
For any $0<c_1\le 1$,
there are
 absolute positive constants
$c_2, c_3$ such that
the following holds.
For any $0 \leq \delta \leq c_2$, 
for any
$(k,p,n,\epsilon,\sigma)$ 
such that $\sigma \geq c_1 \lv \theta^\star \rv_2$, $p \geq n + k$, 
$n \geq \log^2(1/\delta) + k^{1+c_1}$, 
with probability
$1 -4 \delta$ over 
$n$ random draws from $P$,
the output $\htheta$ of $\BP$ satisfies
\begin{align*} 
    R(\htheta) \geq \frac{c_3 \sigma^2}{\log^2 (3p/n)}.
\end{align*}
\end{theorem}
This theorem is proved by showing that the output of basis pursuit is always $n$-sparse and then by simply invoking the previous general result. 
Theorem~\ref{t:bp} implies that the excess risk of
$\BP$ can be much worse than $\mathsf{OLS}$.  For example, if
$k = 5$, $p = n^2$, $\epsilon = 1/n^2$, $\sigma^2 = \lv \theta^\star \rv_2^2 = 1$,
then Theorem~\ref{t:bp} implies an $\Omega\left(\frac{1}{\log^2 n}\right)$
lower bound for $\BP$
where a $O\left(\frac{1}{n}\right)$ upper bound holds for
$\mathsf{OLS}$ \citep{bartlett2020benign,negrea2020defense,tsigler2020benign}. 
If instead,
$k = 5$, $p = n^2$, $\epsilon = 1/n^2$, $\sigma^2 = \lv \theta^\star \rv_2^2 = \log^2 n$, then the excess risk of $\mathsf{OLS}$ goes to zero at a
$O\left(\frac{\log^2 n}{n}\right)$ rate, but Theorem~\ref{t:bp} implies
that the excess risk of $\BP$ is bounded below by a constant.

When $\epsilon=1$, that is, when the covariates are isotropic, our lower bound coincides with the lower bound derived previously by \citet{muthukumar2020harmless}.

\section{Proof of Theorem~\ref{t:main}} \label{s:proof.main}

This section is devoting to proving Theorem~\ref{t:main}, so the
assumptions of Theorem~\ref{t:main} are in scope throughout this
section.  Our proof proceeds through a series of lemmas.  

\begin{definition}
\label{d:S}
For any $v \in \R^p$ 
and any $S \ss [p]$, let
$v_S$ be the vector obtained by
selecting the components of $S$ from $v$ in order.
For $X \in \R^{n \times p}$, define $X_S$ similarly, except selecting
columns from $X$.  Let $H := \{ 1,\ldots,k \}$
and $T := \{ k+1,\ldots, p \}$, so that
$v_H = (v_1,\ldots,v_k)$ and
$v_T = (v_{k+1},\ldots,v_p)$.
\end{definition}

The first step is to break up the excess risk into
contributions from the 
``head'' $H$ and
the ``tail'' $T$.
\begin{lemma}
\label{l:by.norm}The excess risk of any 
parameter vector
$\htheta$ satisfies
\begin{align*}
R(\htheta) = \lv \theta_H^* - \htheta_H \rv_2^2
           + \epsilon \lv \htheta_T \rv_2^2.
\end{align*}
\end{lemma}
\begin{proof}
By the projection lemma,
\begin{align*}
R(\htheta) & = \E_{(x,y) \sim P}[(\htheta \cdot x - y)^2] - \E_{(x,y) \sim P}[(\theta^{\star } \cdot x - y)^2] \\
  & = \E_{(x,y) \sim P}[(\htheta \cdot x - \theta^\star \cdot x)^2] + \sigma^2 - \sigma^2 \\
  & = \lv \theta_H^* - \htheta_H \rv_2^2
           + \epsilon \lv \htheta_T \rv_2^2, 
\end{align*}
since $\theta_T^* = 0$ and the distribution of $x$ has covariance
$\diag(\overbrace{1,1,\ldots,1}^k, \overbrace{\epsilon,\ldots,\epsilon}^{p-k})$.
\end{proof}

Lemma~\ref{l:by.norm} leads to the subproblem
of establishing a lower bound on $\lv \htheta_T \rv_2^2$.
The following lemma is an easy step in this direction.
\begin{lemma}
\label{l:residual.lower}Given any estimator $\htheta$ such that $X\htheta=y$ we have
\begin{align*}
\lv X_T \htheta_T \rv_2
 = \lv y - X_H \htheta_H \rv_2.
\end{align*}
\end{lemma}
\begin{proof}
The lemma follows from the fact that $y = X \htheta = X_H \htheta_H + X_T \htheta_T$.
\end{proof}

Lemma~\ref{l:residual.lower} provides a
means to establish a lower bound on 
$\lv X_T \htheta_T \rv_2$.  This in turn 
can lead to a lower bound on
$\lv \htheta_T \rv_2$ if we can show
that
the linear operator associated
with $X_T$ does not ``blow up'' $\htheta_T$.
It turns out, when $\htheta$
(and thus $\htheta_T$) is sparse, 
a random $X_T$ is especially unlikely to
``blow up'' $\htheta_T$, as reflected in
the following lemma. It is an immediate consequence of~\citep[][Theorem~4.2]{adamczak2012chevet}.
\begin{lemma}
\label{l:sparse.blowup}
There exists a constant $c$ such that for 
any
$t\ge 1$, we have
\begin{align*}
    \Pr\left[
    \max_{S \ss T: |S|\le s } \lv X_S\rv_{op}
    \ge c \sqrt{\epsilon} \left(
         \sqrt{s} \ln \left( \frac{3 (p-k)}{s} \right)
          + \sqrt{n} 
          + t
          \right)
    \right]
    \le e^{-t}.
\end{align*}
\end{lemma}
Lemma~\ref{l:sparse.blowup} implies a lower bound on
$\lv \theta_T\rv_2$ for any $s$-sparse $\theta$.  
\begin{lemma}
\label{l:X_htheta_upper}
There exists a constant $c$ such that, with probability at least 
$1 - \delta$,
any $s$-sparse $\theta$ has
\begin{align}
\label{e:X_htheta_upper}
\lv \theta_T\rv_2
 & \ge
  \frac{\lv X_T \theta_T \rv_2}
       {c \sqrt{\epsilon} \left(
         \sqrt{s} \ln \left( \frac{3 (p-k)}{s} \right)
          + \sqrt{n} 
          + \log(1/\delta)
          \right)}.
\end{align}
\end{lemma}
\begin{proof}
For any $s$-sparse $\theta$, if 
$S = T \cap \{ i : \theta_i \neq 0 \}$,
we have
$|S| \leq s$, so
for any $X$, 
we have  $\lv X_T \theta_T \rv_2 = \lv X_S \theta_S \rv_2$.
Applying Lemma~\ref{l:sparse.blowup} with $t = \log(1/\delta)$, 
with probability at least $1 - \delta$,
\eqref{e:X_htheta_upper} holds
for all such $\theta$.
\end{proof}

Since $\htheta$ is likely to be
$s$-sparse, Lemma~\ref{l:X_htheta_upper}
implies 
a
high-probability lower bound on $\lv \htheta_T\rv_2$, 
the contribution
of the tail to the excess risk.  This bound is in terms of
$\lv X_T \htheta_T \rv_2 = \lv y - X_H \htheta_H \rv_2$.
We will bound this by proving a high-probability
lower bound on  
$\lv y \rv_2$, and a high-probability upper bound on
$\lv X_H \htheta_H \rv_2$.  We start with a lower bound on $\lv y \rv_2$.
\begin{lemma}
\label{l:y.lower}
With probability $1 - \delta$,
\begin{align*}
\lv y \rv_2^2 \geq 
(\sigma^2 +\lv \theta^\star\rv_2^2) 
n \left(1-2\sqrt{\frac{\log(1/\delta)}{n}}\right).
\end{align*}
\end{lemma}
\begin{proof}
We have 
$
    y = X\theta^\star +\xi.
$
That is, for each sample $i \in [n]$, $y_i =x_i \cdot \theta^\star+\xi_i$. 

Observe that $x_i \cdot \theta^\star \sim \cN(0,\lv\theta^\star\rv_2^2)$, $\xi_i \sim \cN(0,\sigma^2)$,
and $x_i$ and $\xi_i$ are independent. Therefore, we have that
$y_i \sim \cN(0,\sigma^2+\lv \theta^\star\rv_2^2)$. 
Thus
\begin{align}
    \lv y\rv_2^2 &= \sum_{i=1}^n |y_i|^2 = \left(\sigma^2+\lv \theta^\star\rv_2^2\right) q, \label{e:prechi_squared_lowerbound}
\end{align}
where $q$ is a 
random variable with a $\chi^2(n)$
distribution.
Applying Lemma~1 from \citep{laurent2000adaptive},
we have
\begin{align*}
    \Pr\left(q\ge n-2\sqrt{tn}\right)\ge 1-\exp(-t).
\end{align*}
If we set $t = \log(1/\delta)$ then with probability at least $1-\delta$,
\begin{align*}
q \ge n\left(1-2\sqrt{\frac{\log(1/\delta)}{n}}\right).
\end{align*}
This combined with \eqref{e:prechi_squared_lowerbound} completes the proof.
\end{proof}

Recall that we also want an upper bound on
$\lv X_H \htheta_H \rv_2$; we will use
a bound that is an immediate consequence of \citep[][Corollary~5.35]{vershynin2010introduction}.
\begin{lemma}
\label{l:X_op}
With probability $1 - \delta$,
\begin{align*}
\lv X_H \rv_{op} \leq \sqrt{n}+\sqrt{k}+\sqrt{2\log(2/\delta)}.
\end{align*}
\end{lemma}

Armed with these lemmas, we can now prove Theorem~\ref{t:main}.

\begin{proof} 
{\bf of Theorem~\ref{t:main}}
With foresight, 
set $\zeta = \sqrt{\frac{c_4 \sigma^2 n}{s \log^2 (3p/s)}}$
for a constant $c_4$ that will be determined by the analysis.

\textbf{Case 1 ($\lv \htheta_{H}\rv_2 \ge \lv \theta^\star \rv_{2}+\zeta$).}  Recall that $\lv \theta^\star\rv_2 = \lv \theta^\star_{H}\rv_2$, since it is zero for all entries after the $k$th coordinate. 
By Lemma~\ref{l:by.norm} we have 
\begin{align*} 
    R(\htheta) & \ge \lv \htheta_{H}-\theta^{\star}_{H}\rv_{2}^2  \ge \left(\lv \htheta_{H}\rv_2 - \lv \theta^\star\rv_2\right)_{+}^2  \ge \zeta^2 = \frac{c_4 \sigma^2 n}{s\log^2 (3p/s)}.
\end{align*}

\textbf{Case 2 ($\lv \htheta_{H}\rv_2 \le \lv \theta^\star \rv_{2}+\zeta$).} 
By Lemma~\ref{l:by.norm}, we have
\begin{align*}
\label{e:by.norm}
R(\htheta) \geq \epsilon\lv \htheta_T \rv_2^2.
\end{align*}
The estimator $\htheta$ is $s$-sparse with probability at least $1-\delta$. Hence, combining Lemmas~\ref{l:residual.lower} and \ref{l:X_htheta_upper}, and taking a union bound we get that, for an absolute positive constant
$c$, with probability
$1 - 2 \delta$,
\begin{align*}
R(\htheta)
 & \geq 
 c \frac{\lv y - X_H \htheta_H \rv_2^2}
        {s \log^2\left(\frac{3(p-k)}{s}\right)+n+\log^2(1/\delta)}  \geq 
 c \frac{\left(\lv y \rv_2 - \lv X_H \htheta_H \rv_2\right)_+^2}
          {s \log^2\left(\frac{3(p-k)}{s}\right)+n+\log^2(1/\delta)} .
\end{align*}
Applying Lemma~\ref{l:y.lower}, 
we find that with probability
at least $1 - 3 \delta$, 
\begin{align*}
  R(\htheta)
 & \geq  
  c \frac{\left[
  \sqrt{(\sigma^2 +\lv \theta^\star\rv_2^2) n \left(1-2\sqrt{\frac{\log(1/\delta)}{n}}\right)}
  - \lv X_H \htheta_H \rv_2\right]_+^2}
          {s \log^2\left(\frac{3(p-k)}{s}\right)+n+\log^2(1/\delta)}.
\end{align*}
Next, by applying Lemma~\ref{l:X_op}, with probability
at least $1 - 4 \delta$,
\begin{align*}
 & R(\htheta) \\
 & \geq  
  c \frac{\left[
  \sqrt{(\sigma^2 +\lv \theta^\star\rv_2^2) n \left(1-2\sqrt{\frac{\log(1/\delta)}{n}}\right)}
  - (\sqrt{n} + \sqrt{k} + \sqrt{2 \log(2/\delta)})
      \lv \htheta_{H} \rv_2
      \right]_+^2}
          {s \log^2\left(\frac{3(p-k)}{s}\right)+n+\log^2(1/\delta)}\\
& \geq  
  c \frac{\left[
  \sqrt{(\sigma^2 +\lv \theta^\star\rv_2^2) n \left(1-2\sqrt{\frac{\log(1/\delta)}{n}}\right)}
  - (\sqrt{n} + \sqrt{k} + \sqrt{2 \log(2/\delta)})
      \left(\lv \theta^\star\rv_2+\zeta\right)
      \right]_+^2}
          {s \log^2\left(\frac{3(p-k)}{s}\right)+n+\log^2(1/\delta)}\\
&=
  c \frac{\left[
  \sqrt{(\sigma^2 \!+\!\lv \theta^\star\rv_2^2) n
  \hspace{-2pt}
  \left(1-2\sqrt{\frac{\log(1/\delta)}{n}}\right)}
  \!-\! (\sqrt{n} \!+\! \sqrt{k} \!+\! \sqrt{2 \log(2/\delta)})
   \hspace{-2pt}
      \left(\lv \theta^\star\rv_2\!+\!\sqrt{\frac{c_4 \sigma^2 n}{s\log^2(3p/s)}}\right)
      \right]_+^2}
          {s \log^2\left(\frac{3(p\!-\!k)}{s}\right)\!+\!n\!+\!\log^2(1/\delta)}.
\end{align*}
By choosing $c_2 > 0$ to be small enough, we
can choose $n$ to be as large as desired.  
Recall that, as $n \rightarrow \infty$, 
both
$k = o(n)$ and $\log(2/\delta) = o(n)$.
Thus with probability at least $1 - 4 \delta$, we have 
that
\begin{align*}
   R(\htheta)
   & \geq (c/2) \frac{\left[
        \sqrt{\sigma^2 + \lv \theta^\star \rv_2^2}
            - \left( 1 + \frac{c_1^2}{8} \right) \left(1+\sqrt{\frac{c_4 \sigma^2 n}{\lv \theta^\star\rv_2^2 s\log^2 (3p/s) }}\right)\lv \theta^\star\rv_2
                   \right]_{+}^2 n
                 }
          {s \log^2\left(\frac{3(p-k)}{s}\right)+n+\log^2(1/\delta)}. 
\end{align*}
Since $s \geq n$, we have
\begin{align*}
   R(\htheta)
   & \geq c' \frac{\left[
        \sqrt{\sigma^2 + \lv \theta^\star \rv_2^2}
            - \left( 1 + \frac{c_1^2}{8} \right) \left(1+\sqrt{\frac{c_4 \sigma^2}{\lv \theta^\star\rv_2^2\log^2 (3p/s) }}\right)\lv \theta^\star\rv_2
                   \right]_{+}^2 n
                 }
          {s \log^2 (3p/s)} \\
   & = c' \frac{\left[
        \sqrt{1 + \frac{\lv \theta^\star \rv_2^2}{\sigma^2}}
            - \left( 1 + \frac{c_1^2}{8} \right) \left(1+\sqrt{\frac{c_4 \sigma^2}{\lv \theta^\star\rv_2^2\log^2 (3p/s) }}\right) \frac{\lv \theta^\star \rv_2}{\sigma}
                   \right]_{+}^2 \sigma^2 n
                 }
          {s \log^2 (3p/s)}.
\end{align*}
Defining $r := \frac{\lv \theta^\star \rv_2}{\sigma}$
and simplifying, we have
\[ 
   R(\htheta)
   \geq c' \frac{\left[
        \sqrt{1 + r^2}
            - \left( 1 + \frac{c_1^2}{8} \right) r
            - \left( 1 + \frac{c_1^2}{8} \right) \sqrt{\frac{c_4}{\log^2 (3p/s) }}
                   \right]_{+}^2 \sigma^2 n
                 }
          {s \log^2 (3p/s)}.
\]
Since
$\sqrt{1 + r^2} - \left( 1 + \frac{c_1^2}{8} \right) r$
is a decreasing function of $r$, 
and,
by assumption, 
$r = \frac{\lv \theta^\star\rv}{\sigma}\leq 1/c_1$, we
have
\[
   R(\htheta)
   \geq c' \frac{\left[
        \sqrt{1 + \frac{1}{c_1^2}}
            - \frac{1}{c_1}
            - \frac{c_1}{8}
            - \left( 1 + \frac{c_1^2}{8} \right) \sqrt{\frac{c_4}{\log^2 (3p/s) }}
                   \right]_{+}^2 \sigma^2 n
                 }
          {s \log^2 (3p/s)}.
\]
Since $s\le p$, we have
\[
   R(\htheta)
   \geq c' \frac{\left[
        \sqrt{1 + \frac{1}{c_1^2}}
            - \frac{1}{c_1}
            - \frac{c_1}{8}-\left(1+\frac{c_1^2}{8}\right)\sqrt{c_4}
                   \right]_{+}^2 \sigma^2 n
                 }
          {s \log^2 (3p/s)}.
\]
Recall that  $0 < c_1 \leq 1$, and 
choose 
\[
c_3 =  \min\left\{c'\left(
        \sqrt{1 + \frac{1}{c_1^2}}
            - \frac{1}{c_1}
            - \frac{c_1}{8}-\left(1+\frac{c_1^2}{8}\right)\sqrt{c_4}
                   \right)_{+}^2,c_4\right\}. 
\]
Thus, if $c_4$ is chosen to be a sufficiently small positive constant then
                   $$c_3 \ge  \min\left\{c'\left(
        \sqrt{1 + \frac{1}{c_1^2}}
            - \frac{1}{c_1}
            - \frac{c_1}{4}
                   \right)_{+}^2,c_4\right\}>0$$ completing the proof.
\end{proof}

\section{Proof of Theorem~\ref{t:bp}} \label{s:proof.bp}

This section is devoted to proving Theorem~\ref{t:bp}, so the
assumptions of Theorem~\ref{t:bp} are in scope throughout this
section.  As in Section~\ref{s:proof.main}, 
our proof proceeds through a series of lemmas. 

The first lemma is an immediate consequence of
\citep[][Proposition 1]{schneider2020geometry}.
\begin{lemma}
\label{l:unique}
Almost surely, there is a unique minimizer of $\lv \theta \rv_1$ subject to $X \theta = y$.
\end{lemma}

The following lemma appears to be known
\citep{chen2001atomic}; we included a proof in
Appendix~\ref{a:bp.sparse} because
we did not find one that applies in our setting.
\begin{lemma}
\label{l:bp.sparse}
Almost surely, the output $\htheta$ of $\BP$ is $n$-sparse.
\end{lemma}

\begin{proof}
{\bf of Theorem~\ref{t:bp}}
By Lemma~\ref{l:bp.sparse}, 
for any $\delta > 0$,
$\BP$ is a
$(n,\delta)$-sparse interpolator for the 
$(k,p,n,\epsilon,\sigma)$-scenario $P$.
Applying Theorem~\ref{t:main} with
$s = n$ completes the proof.
\end{proof}

\section{Discussion}\label{s:discussion}
We have demonstrated that for interpolating linear regression with Gaussian data,
outputting a sparse parameter vector can be harmful, even when learning
a sparse target.

Our proofs only use a few of the properties of
Gaussian distributions, so
in the case that the covariance is
$\Sigma_{k,\epsilon}$, our results should generalize 
to sub-Gaussian and log-concave distributions.  
We chose to analyze $\Sigma_{k,\epsilon}$ because it
is arguably the canonical case where OLS enjoys benign
overfitting, and it leads to clean and interpretable
bounds.
Handling
a wider variety
of covariance matrices
is another very natural future direction.
A starting point would be to generalize \citep[][Theorem 4.2]{adamczak2012chevet}.

Recent research has shown that
linear models parameterized by simple two-layer linear
neural networks with diagonal weight matrices leads to implicit regularization that
interpolates between the
$\ell_1$-norm used by basis pursuit and the $\ell_2$-norm used by
$\mathsf{OLS}$
\citep{pmlr-v125-woodworth20a,azulay2021implicit}.  The connection
of this work to neural networks, together with the stark difference between
$\ell_1$ and $\ell_2$ regularization in the context of benign overfitting highlighted in
this paper, 
motivates the study of benign overfitting with these models.  
(We thank Olivier Bousquet for suggesting this last problem.)

 \printbibliography

\appendix 

\section{Proof of Lemma~\protect\ref{l:bp.sparse}}
\label{a:bp.sparse}

By Lemma~\ref{l:unique}, we may assume without loss of
generality that there is a unique minimizer of $\lv \theta \rv_1$ subject to $X \theta = y$.
Assume for the sake of contradiction that $\htheta$
has
$\lv \htheta \rv_0 = s > n$.  

Let 
\[
I := \{ i \in [p] : \htheta_i \neq 0 \}.
\]
Since $|I| > n$, the columns in $\{ X_i : i \in I \}$ are
linearly dependent.  That is, there exists
a set of weights
$\{ \lambda_i : i \in I \}$,
at least
one of which is nonzero,
such that 
\begin{align}
\label{e:dependent}
\sum_{i \in I} \lambda_i X_i = 0.
\end{align}
Let the vector $\lambda \in \R^p$ be obtained by filling in $\lambda_i = 0$ for $i \notin I$.

From here, we will divide our analysis into cases.

{\bf Case 1} 
($\sum_{i \in I} \lambda_i \sign(\htheta_i) > 0$).
We will prove by contradiction that this case cannot happen.  For an
$\eta > 0$ to be set later, consider
\[
v = \htheta - \eta \sum_{i \in I} \lambda_i e_i = \htheta -\eta \lambda.
\]
First, note that
\[
X v = X \htheta - \eta \sum_{i \in I} \lambda_i X_i = y - 0 = y.
\]
We will now prove the claim that, for a small enough $\eta$, $\lv v \rv_1 < \lv \htheta \rv_1$. This will lead to the desired contradiction.
To establish this claim, it suffices to prove that $\frac{-\lambda}{\lv \lambda \rv_2}$ is a descent direction for
$\lv \cdot \rv_1$ at $\htheta$.  Toward this end, consider
an arbitrary member
$z$ of the 
subgradient of $\lv \cdot \rv_1$ at 
$\htheta$.  Recalling that $\lambda_i = 0$ when $\htheta_i = 0$, we have
\begin{align*}
\lambda \cdot z
 &=\sum_{i \in I} \lambda_i \sign(\htheta_i)  > 0,
\end{align*}
by the assumption of this case.
This implies that $\frac{-\lambda}{\lv \lambda\rv_2}$ is indeed a descent direction, so that, for
a small enough $\eta$, $\lv v \rv_1 < \lv \htheta \rv_1$.
Recalling
that $X v = y$ then yields a contradiction.

{\bf Case 2} 
($\sum_{i \in I} \lambda_i \sign(\htheta_i) < 0$).  This case 
leads to a contradiction symmetrically to the proof in Case 1, using
\[
v = \htheta + \eta \lambda.
\]

{\bf Case 3}
($\sum_{i \in I} \lambda_i \sign(\htheta_i) = 0$).
Choose $i_0$ arbitrarily from among those $i \in I$ such that
$\lambda_i \neq 0$ with the minimum values of $| \lambda_i|$, that is, $i_0 \in \argmin_{i \in [p]}\left\{|\lambda_i|: \lambda_{i}> 0\right\}$.
As in the first case, suppose that $\lambda_{i_0} > 0$
(the other case can be handled symmetrically).
Set
$\eta = \theta_{i_0}/\lambda_{i_0}$, and once again consider the vector
\[
v = \htheta - \eta \lambda.
\]
As before, for all $\eta' \in [0,\eta]$,
$X (\htheta - \eta' \lambda) =y$.  Furthermore, since
$i_0 \in \argmin_{i \in [p]}\left\{|\lambda_i|: \lambda_{i}> 0\right\}$, for
each such $\eta'$, for all $i$,
\[
\sign((\htheta - \eta' \lambda)_i) = \sign(\htheta_i).
\]
This means that along the path from $\htheta$ to $v$,
any 
subgradient $z$ of the $\ell_1$ norm satisfies
\begin{align*}
    z_i = \sign(\htheta_i),\quad \text{for all }i \in I.
\end{align*}
But this means, throughout this path, $\lambda$ is orthogonal to any
subgradient, which in turn means that the $\ell_1$-norm is unchanged.
When $\eta' = \eta$, we have an interpolator with the
same $\ell_1$-norm as $\htheta$ but one fewer nonzero
component, a contradiction.


\end{document}